\pdfoutput=1

\documentclass[11pt]{article}

\usepackage[]{emnlp2021}

\usepackage{times}
\usepackage{latexsym}
\usepackage{graphicx}

\usepackage[T1]{fontenc}

\usepackage[utf8]{inputenc}

\usepackage{microtype}
\usepackage{booktabs}
\usepackage{multirow}
\usepackage{pifont}
\usepackage{balance}
\usepackage{bm}
\usepackage{algorithm}
\usepackage{algorithmicx}
\usepackage{algpseudocode}
\usepackage{amsmath}
\usepackage{makecell}
\usepackage{fancyhdr}
\usepackage{longtable}
\usepackage{amssymb}
\usepackage{stmaryrd,pict2e,picture}

\newtheorem{theorem}{Theorem}
\newenvironment{proof}{{\noindent\it Proof:\quad}}{\hfill $\square$\par}
\title{From Alignment to Assignment: \\Frustratingly Simple Unsupervised Entity Alignment}

\author{Xin Mao$^{1}$, Wenting Wang$^{2}$, Yuanbin Wu$^{1}$, Man Lan$^{1^*}$\\
  $^{1}$ School of Computer Science and Technology, East China Normal University\\
  Lazada, Alibaba Group \\
  \texttt{xmao@stu.ecnu.edu.cn,wenting.wang@lazada.com}\\\texttt{\{ybwu,mlan\}@cs.ecnu.edu.cn}}

\begin{document}
\maketitle
\begin{abstract}
Cross-lingual entity alignment (EA) aims to find the equivalent entities between cross-lingual KGs (Knowledge Graphs), which is a crucial step for integrating KGs.
Recently, many GNN-based EA methods are proposed and show decent performance improvements on several public datasets.
However, existing GNN-based EA methods inevitably inherit poor interpretability and low efficiency from neural networks.
Motivated by the isomorphic assumption of GNN-based methods, we successfully transform the cross-lingual EA problem into an assignment problem.
Based on this re-definition, we propose a frustratingly Simple but Effective Unsupervised entity alignment method (SEU) without neural networks.
Extensive experiments have been conducted to show that our proposed unsupervised approach even beats advanced supervised methods across all public datasets while having high efficiency, interpretability, and stability.
\end{abstract}

\section{Introduction}
\label{sec:intro}
The knowledge graph (KG) represents a collection of interlinked descriptions of real-world objects and events, or abstract concepts (e.g., documents), which has facilitated many downstream applications, such as recommendation systems \cite{DBLP:conf/www/0003W0HC19,DBLP:conf/www/WangZZLXG19} and question-answering \cite{DBLP:conf/www/ZhaoXQB20,DBLP:conf/wsdm/QiuWJZ20}.
Over recent years, a large number of KGs are constructed from different domains and languages by different organizations.
These cross-lingual KGs usually hold unique information individually but also share some overlappings.
Integrating these cross-lingual KGs could provide a broader view for users, especially for the minority language users who usually suffer from lacking language resources.
Therefore, how to fuse the knowledge from cross-lingual KGs has attracted increasing attentions.

\begin{figure}
  \centering
  \includegraphics[width=1\linewidth]{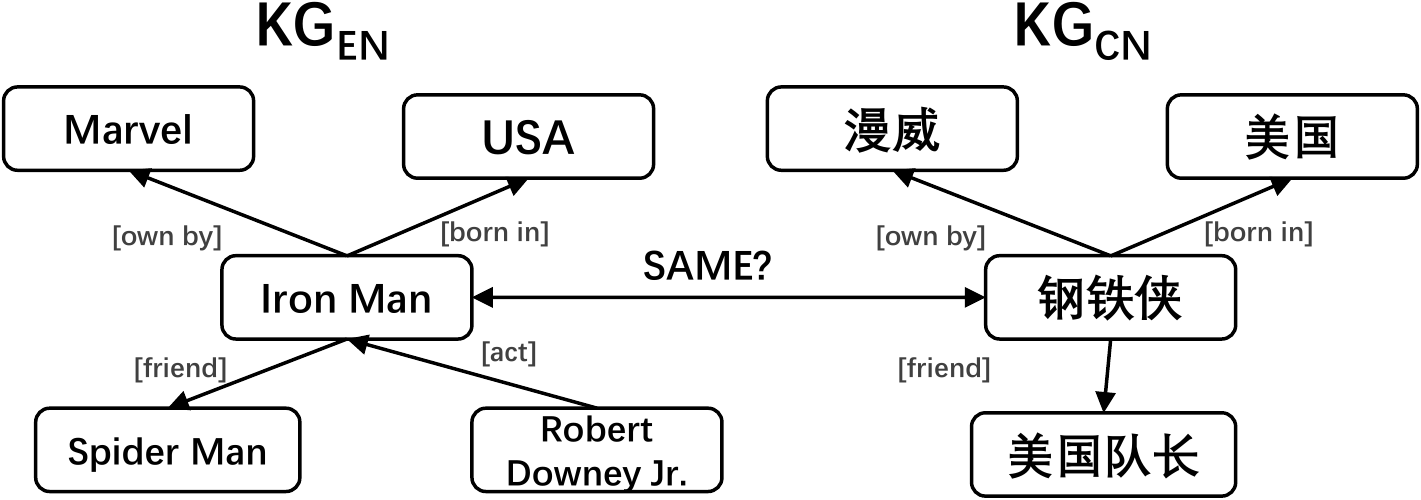}
  \caption{An example of cross-lingual knowledge graph entity alignment.}\label{fig:intro}
\end{figure}

As shown in Figure \ref{fig:intro}, cross-lingual entity alignment (EA) aims to find the equivalent entities across multi-lingual KGs, which is a crucial step for integrating KGs.
Conventional methods \cite{DBLP:journals/pvldb/SuchanekAS11,DBLP:conf/semweb/Jimenez-RuizG11} usually solely rely on lexical matching and probability reasoning, which requires machine translation systems to solve cross-lingual tasks.
However, existing machine translation systems are not able to achieve high accuracy with limited contextual information, especially for language pairs that are not alike, such as Chinese-English and Japanese-English.

Recently, Graph Convolutional Network (GCN) \cite{DBLP:conf/iclr/KipfW17} and subsequent Graph Neural Network (GNN) variants have achieved state-of-the-art results in various graph application.
Intuitively, GNN is better in capturing structural information of KGs to compensate for the shortcoming of conventional methods.
Specifically, several GNN-based EA methods \cite{DBLP:conf/acl/XuWYFSWY19,DBLP:conf/ijcai/WuLF0Y019,DBLP:conf/emnlp/WangYY20} indeed demonstrate decent performance improvements on public datasets.
All these GNN-based EA methods are built upon a core premise, i.e., entities and their counterparts have similar neighborhood structures.
However, better performance is not the only outcome of using GNN.
Existing GNN-based methods inevitably inherit the following inborn defects from neural networks:

(1) \textbf{Poor Interpretability}:
Recently, many researchers view GNN \cite{DBLP:conf/acl/XuWYFSWY19,DBLP:conf/ijcai/WuLF0Y019} as a black box, focusing on improving performance metrics.
The tight coupling between nonlinear operations and massive parameters makes GNN hard to be interpreted thoroughly.
As a result, it is hard to judge whether the new designs are universal or just over-fitting on a specific dataset.
A recent summary \cite{DBLP:conf/coling/ZhangLCCLXZ20} notes that several "advanced" EA methods are even beaten by the conventional methods on several public datasets.

(2) \textbf{Low Efficiency}:
To further increase the performance, newly proposed EA methods try to stack novel techniques, e.g., Graph Attention Networks \cite{DBLP:conf/ijcai/WuLF0Y019}, Graph Matching Networks \cite{DBLP:conf/acl/XuWYFSWY19}, and Joint Learning \cite{DBLP:conf/acl/CaoLLLLC19}.
Consequently, the overall architectures become more and more unnecessarily complex, resulting in their time-space complexities also dramatically increase.
\citet{9174835} present that the running time of complex methods (e.g., RDGCN \cite{DBLP:conf/ijcai/WuLF0Y019}) is $10\times$ more than that of vanilla GCN \cite{DBLP:conf/emnlp/WangLLZ18}.

In this paper, we notice that existing GNN-based EA methods inherit considerable complexity from their neural network lineage.
Naturally, we consider eliminating the redundant designs from existing EA methods to enhance interpretability and efficiency without losing accuracy.
Leveraging the core premise of GNN-based EA methods, we re-state the assumption that both structures and textual features of source and target KGs are isomorphic.
With this assumption, we are able to successfully transform the cross-lingual EA problem into an assignment problem, which is a fundamental and well-studied combinatorial optimization problem.
Afterward, the assignment problem could be easily solved by the Hungarian algorithm \cite{kuhn1955hungarian} or Sinkhorn operation \cite{DBLP:conf/nips/Cuturi13}.

Based on the above findings, we propose a frustratingly Simple but Effective Unsupervised EA method (SEU) without neural networks.
Compared to existing GNN-based EA methods, SEU only retains the basic graph convolution operation for feature propagation while abandoning the complex neural networks, significantly improving efficiency and interpretability.
Experimental results on the public datasets show that SEU could be completed in several seconds with the GPU or tens of seconds with the CPU.
More startlingly, our unsupervised method even outperforms the state-of-the-art supervised approaches across all public datasets.
Furthermore, we discuss the possible reasons behind the unsatisfactory performance of existing complex EA methods and the necessity of neural networks in cross-lingual EA.
The main contributions are summarized as follows:

\begin{itemize}
  \item By assuming that both structures and textual features of source and target KGs are isomorphic, we successfully transform the cross-lingual EA problem into an assignment problem.
Based on this finding, we propose a frustratingly Simple but Effective Unsupervised entity alignment method (SEU).
  \item Extensive experiments on public datasets indicate that our unsupervised method outperforms all advanced supervised competitors while preserving high efficiency, interpretability, and stability.
\end{itemize}

\section{Task Definition}
\label {sec:TF}
KG stores the real-world knowledge in the form of triples $(h,r,t)$.
A KG could be defined as $G=(E,R,T)$, where $E$, $R$, and $T$ represent the entity set, relation set, and triple set, respectively.
Given a source graph $G_s = (E_s,R_s,T_s)$ and a target graph $G_t = (E_t,R_t,T_t)$, EA aims to find the entity correspondences $\bm P$ between KGs.

\begin{figure}
  \centering
  \includegraphics[width=1\linewidth]{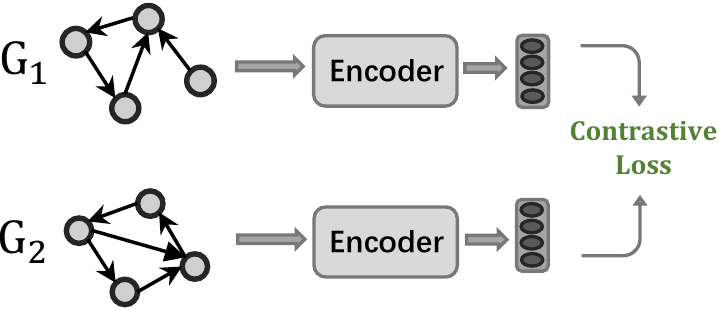}
  \caption{The architecture of existing EA methods.}\label{fig:rw}
\end{figure}

\section{Related Work}
\subsection{Cross-lingual Entity Alignment}
Existing cross-lingual EA methods are based on the premise that equivalent entities in different KGs have similar neighboring structures.
Following this idea, most of them can be summarized into two steps (as shown in Figure \ref{fig:rw}):
(1) Using KG embedding methods (e.g., TransE \cite{DBLP:conf/nips/BordesUGWY13} and GCN \cite{DBLP:journals/corr/KipfW16}) to generate low-dimensional embeddings for entities and relations in each KGs.
(2) Mapping these embeddings into a unified vector space through contrastive losses \cite{DBLP:conf/cvpr/HadsellCL06,DBLP:conf/cvpr/SchroffKP15} and pre-aligned entity pairs.

Based on the vanilla GCN, many EA methods design task-specific modules for improving the performance of EA.
\citet{DBLP:conf/acl/CaoLLLLC19} propose a multi-channel GCN to learn multi-aspect information from KGs.
\citet{DBLP:conf/ijcai/WuLF0Y019} use a relation-aware dual-graph network to incorporate relation information with structural information.
Moreover, due to the lack of labeled data, some methods \cite{DBLP:conf/ijcai/SunHZQ18,DBLP:conf/wsdm/MaoWXLW20} apply iterative strategies to generate semi-supervised data.
In order to provide a multi-aspect view from both structure and semantic, some methods \cite{DBLP:conf/emnlp/WuLFWZ19,DBLP:conf/emnlp/YangZSLLS19} use word vectors of translated entity names as the input features of GNNs.

\begin{figure}
  \centering
  \includegraphics[width=0.9\linewidth]{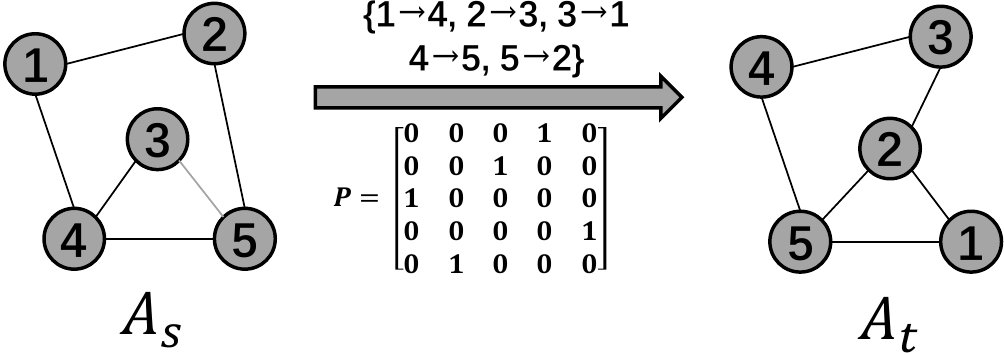}
  \caption{An example of isomorphic graph.}\label{fig:E-A}
\end{figure}

\subsection{Assignment Problem}
\label{rw:ap}
The assignment problem is a fundamental and well-studied combinatorial optimization problem.
An intuitive instance is to assign $N$ jobs for $N$ workers.
Assuming that each worker can do each job at a term, though with varying degrees of efficiency, let $x_{ij}$ be the profit if the $i$-th worker is assigned to the $j$-th job.
Then the problem is to find the best assignment plan (which job should be assigned to which person in one-to-one basis) so that the total profit of performing all jobs is maximum.
Formally, it is equivalent to maximizing the following equation:
\begin{equation}
  \underset{\bm P\in\mathbb{P}_N}{arg\;max}{\;\left\langle \bm P,\bm X\right\rangle}_F
\end{equation}
$\bm X \in \mathbb{R}^{N\times N}$ is the profit matrix.
$\bm P$ is a permutation matrix denoting the assignment plan.
There are exactly one entry of $1$ in each row and each column in $\bm P$ while $0$s elsewhere.
$\mathbb{P}_N$ represents the set of all N-dimensional permutation matrices.
Here, $\langle\cdot\rangle_F$ represents the Frobenius inner product.
In this paper, we adopt the Hungarian algorithm \cite{kuhn1955hungarian} and the Sinkhorn operation \cite{DBLP:conf/nips/Cuturi13} to solve the assignment problem.

\section{The Proposed Method}

\subsection{From Alignment to Assignment}
\label{proof}
The inputs of our proposed SEU are four matrices:
$\bm A_s \in \mathbb{R}^{|E_s|\times|E_s|}$ and $\bm A_t \in \mathbb{R}^{|E_t|\times|E_t|}$ represent the adjacent matrices of the source graph $G_s$ and the target graph $G_t$.
$\bm H_s \in \mathbb{R}^{|E_s|\times d}$ and $\bm H_t\in \mathbb{R}^{|E_t|\times d}$ represent the textual features of entities that have been pre-mapped into a unified semantic space through machine translation systems or cross-lingual word embeddings.

Similar to the assignment plan, aligned entity pairs in EA also needs to satisfy the one-to-one constraint.
Let a permutation matrix $\bm P \in \mathbb{P}_{|E|}$ represent the entity correspondences between $G_s$ and $G_t$.
$\bm P_{ij}=1$ indicates that $e_i\in G_s$ and $e_j\in G_t$ are an equivalent entity pair.
The goal of SEU is to solve $\bm P$ according to $\{\bm A_s, \bm A_t,\bm H_s,\bm H_t\}$.
Consider the following ideal situation:

(1) $\bm A_s$ and $\bm A_t$ are isomorphic, i.e., $\bm A_s$ could be transformed into $\bm A_t$ by reordering the entity node indices according to $\bm P$ (as shown in Figure \ref{fig:E-A}):
\begin{align}
\bm P\bm A_s\bm P^{-1}&=\bm A_t
\end{align}

(2) The textual features of equivalent entity pairs are mapped perfectly by the translation system. Therefore, $\bm H_s$ and $\bm H_t$ could also be aligned according to the entity correspondences $\bm P$:
\begin{align}
\bm P\bm H_s&=\bm H_t
\end{align}

By combining Equation (2) and (3), the connection between the $5$-tuple $\{\bm A_s,\bm A_t,\bm H_s,\bm H_s,\bm P\}$ could be described as follows:
\begin{align}
\begin{split}
    {(\bm P\bm A_s\bm P^{-1})}^l\bm P\bm H_s&=\bm A_t^l\bm H_t \;\;\;\;\forall l\in\mathbb{N} \\
    \Rightarrow \;\;\;\;\;\;\;\;\;\;\;\;\;\;\;\;\;\;{\bm P\bm A_s^l}\bm H_s&=\bm A_t^l\bm H_t \;\;\;\;\forall l\in\mathbb{N}
\end{split}
\end{align}
Based on Equation (4), $\bm P$ could be solved by minimizing the Frobenius norm $\|{\bm P\bm A_s^l}\bm H_s-\bm A_t^l\bm H_t\|^2_F$ under the one-to-one constraint $\bm P \in \mathbb{P}_{|E|}$.
Theoretically, for arbitrarily depth $l \in \mathbb{N}$, the solution of $\bm P$ should be the same.
However, the above inference is based on the ideal isomorphic situation.
In practice, $G_s$ and $G_t$ are not strictly isomorphic and the translation system cannot perfectly map the textual features into a unified semantic space either.
In order to reduce the impact of noise existing in practice, $\bm P$ should be fit for various depths $l$.
Therefore, we propose the following equation to solve the cross-lingual EA problem:
\begin{equation}
    \underset{\bm P\in{\mathbb{P}}_{\vert E\vert}}{\;arg\;min}\sum_{l=0}^L\;\left\|\bm P\bm A_s^l\bm H_s-\bm A_t^l\bm H_t\right\|_F^2
\end{equation}

\begin{theorem}
Equation (5) is equivalent to solving the following assignment problem:
\begin{small}
\begin{equation}
\underset{\bm P\in{\mathbb{P}}_{\vert E\vert}}{\;arg\;max}{\;\left\langle \bm P,\sum_{l=0}^L\bm A_t^l\bm H_t\left(\bm A_s^l\bm H_s\right)^T\right\rangle}_F
\end{equation}
\end{small}
\end{theorem}

\begin{proof}
According to the property of Frobenius norm $\|\bm A-\bm B\|^2_F = \|\bm A\|^2_F + \|\bm B\|^2_F - 2\langle \bm A,\bm B \rangle_F$, Equation (5) could be derived into following:
\begin{small}
\begin{align}
\begin{split}
&\underset{\bm P\in{\mathbb{P}}_{\vert E\vert}}{\;arg\;min}\sum_{l=0}^L\;\left\|\bm P\bm A_s^l\bm H_s-\bm A_t^l\bm H_t\right\|_F^2\\
\;=&\underset{\bm P\in{\mathbb{P}}_{\vert E\vert}}{\;arg\;min}\sum_{l=0}^L\;\left\|\bm P\bm A_s^l\bm H_s\right\|_F^2 + \left\|\bm A_t^l\bm H_t\right\|_F^2\\
\;&-2{\left\langle \bm P \bm A_s^l\bm H_s,\;\bm A_t^l\bm H_t\right\rangle}_F
\end{split}
\end{align}
\end{small}
Here, the permutation matrix $\bm P$ must be orthogonal, so both $\left\|\bm P\bm A_s^l\bm H_s\right\|_F^2$ and $\left\|\bm A_t^l\bm H_t\right\|_F^2$ are constants.
Then, Equation (7) is equivalent to maximizing as below:
\begin{small}
\begin{align}
&\underset{\bm P\in{\mathbb{P}}_{\vert E\vert}}{\;arg\;max}\sum_{l=0}^L \;{\left\langle \bm P\bm A_s^l\bm H_s,\bm A_t^l\bm H_t\right\rangle}_F
\end{align}
\end{small}
For arbitrarily real matrices $\bm A$ and $\bm B$, these two equations always hold: $\langle \bm A,\bm B \rangle_F = {\rm Tr}(\bm A\bm B^T)$ and $\langle \bm A,\bm B+\bm C\rangle_F = \langle \bm A,\bm B\rangle_F + \langle \bm A,\bm C\rangle_F$, where ${\rm Tr}(\bm X)$ represents the trace of matrix $\bm X$.
Therefore, Theorem 1 is proved:

\begin{small}
\begin{align}
\begin{split}
&\underset{\bm P\in{\mathbb{P}}_{\vert E\vert}}{\;arg\;max}\sum_{l=0}^L \;{\left\langle \bm P\bm A_s^l\bm H_s,\bm A_t^l\bm H_t\right\rangle}_F\\
=&\underset{\bm P\in{\mathbb{P}}_{\vert E\vert}}{\;arg\;max}\sum_{l=0}^L \;{\rm Tr}\left( \bm P\bm A_s^l\bm H_s  (\bm A_t^l\bm H_t)^T\right)\\
=&\underset{\bm P\in{\mathbb{P}}_{\vert E\vert}}{\;arg\;max}\sum_{l=0}^L \;{\left\langle \bm P, \bm A_t^l\bm H_t(\bm A_s^l\bm H_s)^T\right\rangle}_F\\
=&\underset{\bm P\in{\mathbb{P}}_{\vert E\vert}}{\;arg\;max}{\;\left\langle \bm P,\sum_{l=0}^L\bm A_t^l\bm H_t\left(\bm A_s^l\bm H_s\right)^T\right\rangle}_F
\end{split}
\end{align}
\end{small}
\end{proof}

By Theorem 1, we successfully transform the EA problem into the assignment problem.
Compared to GNN-based EA methods, our proposed method retains the basic graph convolution operation for feature propagation but replaces the complex neural networks with the well-studied assignment problem.
Note that the entity scales $|E_s|$ and $|E_t|$ are usually inconsistent in practice, resulting in the profit matrix not being a square matrix.
This kind of unbalanced assignment problem could be reduced to the balanced assignment problem easily.
Assuming that $|E_s|$$>$$|E_t|$, a naive reduction is to pad the profit matrix with zeros such that its shape becomes $\mathbb{R}^{|E_s|\times |E_s|}$.
This naive reduction is suitable for the dataset with a small gap between $|E_s|$ and $|E_t|$.
For the dataset with a large entity scale gap, there is a more efficient reduction algorithm available \cite{ramshaw2012minimum}.

\subsection{Two Algorithms for Solving the Assignment Problem}
\label{sec:SAP}
The first polynomial time-complexity algorithm for the assignment problem is the Hungarian algorithm \cite{kuhn1955hungarian}, which is based on improving a matching along the augmenting paths.
The time complexity of the original Hungarian algorithm is $O(n^4)$.
Later, \citet{DBLP:journals/computing/JonkerV87} improve the algorithm to achieve $O(n^3)$ running time, which is one of the most popular variants.

Besides the Hungarian algorithm, the assignment problem could also be regarded as a special case of the optimal transport problem.
In the optimal transport problem, the assignment plan $\bm P$ could be any doubly stochastic matrix instead of a permutation matrix.
Based on the Sinkhorn operation \cite{sinkhorn1964relationship,DBLP:journals/corr/abs-1106-1925}, \citet{DBLP:conf/nips/Cuturi13} proposes a fast and completely parallelizable algorithm for the optimal transport problem:
\begin{align}
\begin{split}
    S^0(\bm X) &= exp(\bm X),\\
    S^k(\bm X) &= {\mathcal N}_c({\mathcal N}_r(S^{k-1}(\bm X))),\\
    {\rm Sinkhorn}(\bm X) &= \lim_{k\rightarrow\infty}S^k(\bm X).
\end{split}
\end{align}
where ${\mathcal N}_r(\bm X)$$=$$ \bm X \varoslash (\bm X \bm 1_{N}\bm 1_{N}^T)$ and ${\mathcal N}_c $$=$$ \bm X \varoslash (\bm 1_{N}\bm 1_{N}^T\bm X)$ are the row and column-wise normalization operators of a matrix, $\varoslash$ represents the element-wise division, and $\bm 1_{N}$ is a column vector of ones.
Then, \citet{DBLP:conf/iclr/MenaBLS18} further prove that the assignment problem could also be solved by the Sinkhorn operation as a special case of the optimal transport problem:
\begin{align}
\begin{split}
  &\underset{\bm P\in\mathbb{P}_N}{arg\;max}{\;\left\langle \bm P,\bm X\right\rangle}_F\\
  =& \lim_{\tau\rightarrow 0^+}{\rm Sinkhorn}(\bm X/\tau)
\end{split}
\end{align}

In general, the time complexity of the Sinkhorn operation is $O(kn^2)$.
Because the number of iteration $k$ is limited, the Sinkhorn operation can only obtain an approximate solution in practice.
But according to our experimental results, a very small $k$ is enough to achieve decent performance in entity alignment.
Therefore, compared to the Hungarian algorithm, the Sinkhorn operation is $n$ times more efficient, i.e., $O(n^2)$.

\subsection{Implementation Details}
The above two sections introduce how to transform the cross-lingual EA problem into the assignment problem and how to solve the assignment problem.
This section will clarify two important implementation details of our proposed method SEU.

\subsubsection{Textual Features $\bm H$}
\label{TF}
The input features of SEU include two aspects:

\textbf{Word-Level}.
In previous cross-lingual EA methods \cite{DBLP:conf/acl/XuWYFSWY19,DBLP:conf/ijcai/WuLF0Y019}, the most commonly used textual features are word-level entity name vectors.
Specifically, these methods first use machine translation systems or cross-lingual word embeddings to map entity names into a unified semantic space and then average the pre-trained entity name vectors to construct the initial features.
To make fair comparisons, we adopt the same entity name translations and word vectors provided by \citet{DBLP:conf/acl/XuWYFSWY19}.

\textbf{Char-Level}.
Because of the contradiction between the extensive existence of proper nouns (e.g., person and city name) and the limited size of word vocabulary, the word-level EA methods suffer from a serious out of vocabulary (OOV) issue.
Therefore, many EA methods explore the char-level features, using char-CNN \cite{DBLP:conf/emnlp/WangYY20} or name-BERT \cite{DBLP:conf/emnlp/LiuCPLC20} to extract the char/sub-word features of entities.
In order to keep the simplicity and consistency of our proposed method, we adopt the character bigrams of translated entity names as the char-level input textual features instead of complex neural networks.

In addition to these text-based methods, we notice that some structure-based EA methods \cite{DBLP:conf/emnlp/WangLLZ18, DBLP:conf/icml/GuoSH19} do not require any textual information at all, where the entity features are randomly initialized.
Section \ref{sec:dis} will discuss the connection between text-based and structure-based methods and challenge the necessity of neural networks in cross-lingual EA.

\subsubsection{Adjacent Matrix $\bm A$}
\label{sec:adj}
In Section \ref{proof}, all deductions are built upon the assertion that the adjacency matrices $\bm A_s$ and $\bm A_t$ are isomorphic.
Obviously, let $\bm D$ be the degree matrix of adjacency matrix $\bm A_{s/t}$, the equal probability random walk matrix $\bm A_r = \bm D^{-1}\bm A_{s/t}$ and the symmetric normalized Laplacian matrix $\bm A_L =\bm I - \bm D^{-1/2}\bm A_{s/t}\bm D^{-1/2}$ of $A_s$ and $A_t$ are also isomorphic too.
Therefore, if $\bm A_{s/t}$ is replaced by $\bm  A_r$ or $\bm A_L$, our method still holds.

However, the above matrices ignore the relation types in the KGs and treat all types of relations equally important.
We believe the relations with less frequency should have higher weight because they represent more unique information.
Following this intuition, we apply a simple strategy to generate the relational adjacency matrix $\bm A_{rel}$, for $a_{ij} \in \bm A_{rel}$:
\begin{small}
\begin{equation}
    \bm a_{ij} = \frac{\sum_{r_j\in R_{i,j}}\ln(\vert T\vert/\vert T_{r_j}\vert)}{\sum_{k\in{\mathcal N}_i}\sum_{r_k\in R_{i,k}}\ln(\vert T\vert/\vert T_{r_k}\vert)}
\end{equation}
\end{small}
where ${\mathcal N}_i$ represents the neighboring set of entity $e_i$, $R_{i,j}$ is the relation set between $e_i$ and $e_j$, $|T|$ and $|T_r|$ represent the total number of all triples and the triples containing relation $r$, respectively.

\section{Experiments}
Our experiments are conducted on a workstation with a GeForce GTX Titan X GPU and a Ryzen ThreadRipper 3970X CPU.
The code and datasets are available in \url{github.com/MaoXinn/SEU}.

\subsection{Datasets}
To make fair comparisons with previous EA methods, we experiment with two widely used public datasets:
(1) \textbf{DBP15K} \cite{DBLP:conf/semweb/SunHL17}:
This dataset consists of three cross-lingual subsets from multi-lingual DBpedia: $\rm DBP_{FR-EN}$, $\rm DBP_{ZH-EN}$, $\rm DBP_{JA-EN}$.
Each subset contains $15,000$ entity pairs.
(2) \textbf{SRPRS}:
\citet{DBLP:conf/icml/GuoSH19} propose this sparse dataset, including two cross-lingual subsets: $\rm SRPRS_{FR-EN}$ and $\rm SRPRS_{DE-EN}$.
Each subset also contains $15,000$ entity pairs but with much fewer triples compared to DBP$15$K.

\begin{table}[t]
\begin{center}
\resizebox{0.9\linewidth}{!}{
\renewcommand\arraystretch{0.95}
\begin{tabular}{p{2cm}c|cccccc}
\toprule
\multicolumn{2}{c|}{Datasets} & $|E|$ & $|R|$  & $|T|$\\
\toprule
\multirow{2}{1.3cm}{$\rm{DBP_{ZH-EN}}$} & Chinese & 19,388 & 1,701& 70,414\\
& English & 19,572 & 1,323 & 95,142 \\
\multirow{2}{1.3cm}{$\rm{DBP_{JA-EN}}$} & Japanese & 19,814 & 1,299 & 77,214\\
& English & 19,780 & 1,153  & 93,484 \\
\multirow{2}{1.3cm}{$\rm{DBP_{FR-EN}}$} & French & 19,661 & 903 & 105,998\\
& English & 19,993 & 1,208 & 115,722  \\
\hline
\multirow{2}{1.3cm}{$\rm{SRPRS_{FR-EN}}$} & French & 15,000 & 177& 33,532\\
& English & 15,000 & 221& 36,508 \\
\multirow{2}{1.3cm}{$\rm{SRPRS_{DE-EN}}$} & German & 15,000 & 120 & 37,377\\
& English & 15,000 & 222 & 38,363  \\
\bottomrule
\end{tabular}
}
\end{center}
\caption{Statistical data of DBP15K and SRPRS.}\label{table:data}
\end{table}

The statistics of these datasets are summarized in Table \ref{table:data}.
Most of the previous studies\cite{DBLP:conf/emnlp/WangLLZ18,DBLP:conf/acl/CaoLLLLC19} randomly split $30\%$ of the entity pairs for training and development, while using the remaining $70\%$ for testing.
Because our proposed method is unsupervised, all of the entity pairs could be used for testing.

\begin{table*}[t]
\begin{center}
\resizebox{1\textwidth}{!}{
\renewcommand\arraystretch{1.47}
\begin{tabular}{c|ccc|ccc|ccc|ccc|ccc}
  \toprule
  {\multirow{2}{*}{Method}} & \multicolumn{3}{c|}{$\rm{DBP_{ZH-EN}}$} & \multicolumn{3}{c|}{$\rm{DBP_{JA-EN}}$} & \multicolumn{3}{c|}{$\rm{DBP_{FR-EN}}$}& \multicolumn{3}{c|}{$\rm{SRPRS_{FR-EN}}$}& \multicolumn{3}{c}{$\rm{SRPRS_{DE-EN}}$}  \\
   &H@1 & H@10 & MRR & H@1 & H@10 & MRR & H@1 & H@10 & MRR & H@1 & H@10 & MRR & H@1 & H@10 & MRR\\
  \hline
   GCN-Align & 0.434 & 0.762 & 0.550 & 0.427 & 0.762 & 0.540 & 0.411 & 0.772 & 0.530 & 0.243& 0.522&0.340&0.385&0.600& 0.460\\
   MuGNN & 0.494 & 0.844 & 0.611 & 0.501 & 0.857 & 0.621 & 0.495 & 0.870  &0.621 &0.131&0.342& 0.208& 0.245&0.431&0.310\\
   BootEA & 0.629 & 0.847 & 0.703 & 0.622 & 0.853 & 0.701 & 0.653 & 0.874 & 0.731 & 0.365&0.649&0.460&0.503&0.732&0.580\\
   MRAEA &0.757&0.930&0.827&0.758&0.934&0.826&0.781&0.948&0.849&0.460&0.768&0.559&0.594&0.818&0.666\\
   JEANS & 0.719&0.895&0.791 & 0.737&0.914&0.798 & 0.769&0.940&0.827&-&-&-&-&-&-\\
  \hline
   GM-Align & 0.679 & 0.785 & - & 0.739 & 0.872 & - & 0.894 & 0.952 & - & 0.574&0.646&0.602&0.681&0.748&0.710\\
   RDGCN & 0.697 & 0.842 & 0.750 & 0.763 & 0.897 & 0.810 & 0.873 & 0.950 & 0.901  &0.672&0.767& 0.710&0.779&0.886&0.820 \\
   HGCN &0.720&0.857&0.760&0.766&0.897&0.810&0.892&0.961&0.910&0.670&0.770&0.710&0.763&0.863&0.801\\
   DAT &-&-&-&-&-&-&-&-&-&0.758&0.899&0.810&0.876&0.955&0.900\\
   DGMC &0.801&0.875&-&0.848&0.897&-&0.933&0.960&-&-&-&-&-&-&-\\
  \hline
   AttrGNN & 0.796&0.929&0.845&0.783&0.920&0.834&0.919&0.979&0.910&-&-&-&-&-&-\\
   CEA & 0.787 & -& - & 0.863 &- & - & 0.972 & - & - & 0.962&-&-&0.971&-&-\\
   EPEA&0.885&0.953&0.911&0.924&0.969&0.942&0.955&0.986&0.967&-&-&-&-&-&-\\
  \hline
   SEU(word) & 0.816 & 0.923& 0.854 & 0.865 &0.952 & 0.896 & 0.953 & 0.989 & 0.967 & 0.812&0.902&0.843&0.902&0.951&0.920\\
   SEU(char) & 0.870 & 0.947 & 0.897 & 0.947 & 0.984 & 0.961 & 0.986 & 0.998 & 0.990 &0.979&0.994& 0.985&0.980&0.994&0.985 \\
   SEU(w+c)&\textbf{0.900}&\textbf{0.965}&\textbf{0.924}&\textbf{0.956}&\textbf{0.991}&\textbf{0.969}&\textbf{0.988}&\textbf{0.999}&\textbf{0.992}&\textbf{0.982}&\textbf{0.995}&\textbf{0.986}&\textbf{0.983}&\textbf{0.996}&\textbf{0.987}\\
  \bottomrule
\end{tabular}
}
\caption{Main experimental results on DBP$15$K and SRPRS. Baselines are separated in accord with the three groups described in Section \ref{baseline}. Most results are from the original papers.
Some recent papers are failed to run on missing datasets or do not release the source code yet. We will fill in these blanks after contacting their authors.}
\label{table:res1}
\vspace{-2em}
\end{center}
\end{table*}

\subsection{Baselines}
\label{baseline}
We compare our method against the following three groups of advanced EA methods:
(1) \textbf{Structure}:
These methods only use the structure information (i.e., triples):
GCN-Align \cite{DBLP:conf/emnlp/WangLLZ18}, MuGNN \cite{DBLP:conf/acl/CaoLLLLC19}, BootEA \cite{DBLP:conf/ijcai/SunHZQ18}, MRAEA \cite{DBLP:conf/wsdm/MaoWXLW20}, JEANS \cite{DBLP:conf/eacl/ChenSZR21}.
(2) \textbf{Word-level}:
These methods average the pre-trained entity name vectors to construct the initial features:
GM-Align \cite{DBLP:conf/acl/XuWYFSWY19}, RDGCN \cite{DBLP:conf/ijcai/WuLF0Y019}, HGCN \cite{DBLP:conf/emnlp/WuLFWZ19}, DAT \cite{DBLP:conf/sigir/Zeng00TT20}, DGMC \cite{DBLP:conf/iclr/FeyL0MK20}.
(3) \textbf{Char-level}:
These EA methods further adopt the char-level textual features:
AttrGNN \cite{DBLP:conf/emnlp/LiuCPLC20}, CEA \cite{DBLP:conf/icde/Zeng0T020}, EPEA \cite{DBLP:conf/emnlp/WangYY20}.

For our proposed method, SEU(word) and SEU(char) represent the model only using the word and char features as the inputs, respectively.
SEU(w+c) represents concatenating the word and char features together as the inputs.

\subsection{Settings}
\textbf{Metrics}.
Following convention, we use $Hits@k$ and \emph{Mean Reciprocal Rank} (MRR) as our evaluation metrics.
The $Hits@k$ score is calculated by measuring the proportion of correct pairs in the top-$k$.
In particular, $Hits@1$ equals accuracy.

\noindent
\textbf{Hyper-parameter}.
In the main experiments, we use the Sinkhorn operation to solve the assignment problem.
For all dataset, we use a same default setting:
the depth $L=2$;
the iterations $k=10$;
the temperature $\tau = 0.02$.

\subsection{Main Experiments}
Table \ref{table:res1} shows the main experimental results of all EA methods.
Numbers in \textbf{bold} denote the best results among all methods.

\textbf{SEU vs. Baselines}.
According to the results, our method consistently achieves the best performance across all datasets.
Compared with the previous SOTA methods, SEU (w+c) improves the performance on $Hits@1$ and $MRR$ by $1.5\%$ and $1.3\%$ at least.
More importantly, SEU outperforms the supervised competitors as an unsupervised method, which is critical in practical applications.

In addition to the better performances, SEU also has better interpretability and stability:
(1) When solving with the Hungarian algorithm, we can trace the reasons for each decision by the augmenting path, which brings better interpretability.
(2) As we all know, neural networks optimized by SGD usually have some performance fluctuations.
Since both the Hungarian algorithm and Sinkhorn operation are deterministic, multiple runs of these algorithms remain unchanged under the same hyper-parameters, which means better stability.

\textbf{Word vs. Char}.
From Table \ref{table:res1}, we observe that the char-level SEU greatly outperforms the word-level SEU.
Especially in $\rm SRPRS_{FR-EN}$, the performance gap on $Hits@1$ is more than $16\%$.
As mentioned in Section \ref{TF}, the main reason is that these datasets contain extensive OOV proper nouns.
For example, in $\rm DBP15K$, $4$-$6\%$ of the words are OOV;
while in $\rm SRPRS_{DE-EN}$ and $ \rm SRPRS_{FR-EN}$, more than $12\%$ and $16\%$ of the entity names are OOV, respectively.

Note that the performance difference between SEU(word) and SEU (char) is vast, but these two features still complement to each other, so the combination of them still improves the performances (especially on $\rm DBP_{ZH-EN}$ dataset).
We believe the hidden reason is synonyms.
For example, \emph{soccer} and \emph{football} refer to the same Chinese phrase, but there is almost no overlap in the char-level between these two English words.
However, the word-level features could bridge such semantic gap via pre-trained cross lingual word vectors.

\begin{figure}[t!]
  \centering
  \includegraphics[width=0.85\linewidth]{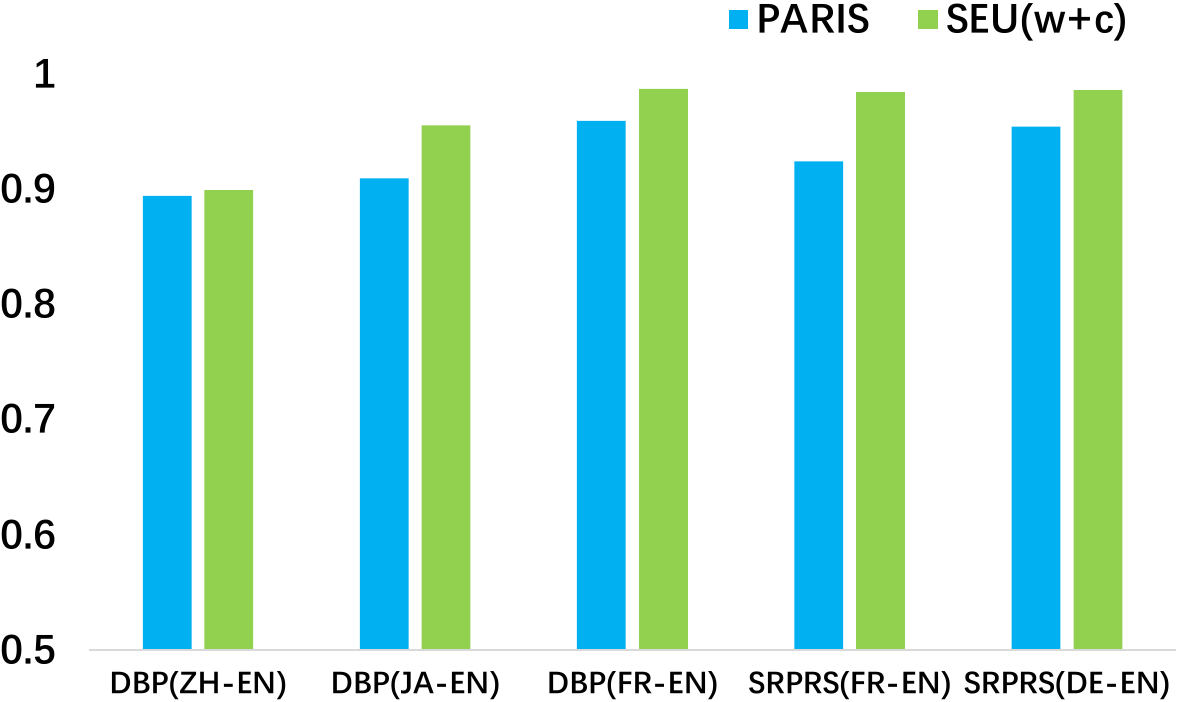}
  \caption{F1-score of SEU(w+c) and PARIS.}\label{fig:pairs}
\end{figure}

\begin{table}[t!]
\begin{center}
\resizebox{0.95\linewidth}{!}{
\renewcommand\arraystretch{0.9}
\begin{tabular}{cccccc}
  \toprule
  algorithm&$\rm DBP_{ZH-EN}$&$\rm DBP_{JA-EN}$&$\rm DBP_{FR-EN}$\\
  \toprule
  Hungarian  &0.907&0.963&0.993\\
  Sinkhorn  &0.900&0.956&0.988\\
  \bottomrule
\end{tabular}
}
\end{center}
\caption{$Hits@1$ of Hungarian and Sinkhorn. \protect\footnotemark }\label{hvsp}
\end{table}
\footnotetext{Since the Hungarian algorithm only outputs the assigned entity pairs, instead of a probability matrix $\bm P$, we can only report the $Hits@1$ performance.}

\begin{table}[t!]
\begin{center}
\resizebox{0.95\linewidth}{!}{
\renewcommand\arraystretch{0.9}
\begin{tabular}{cccccc}
  \toprule
  algorithm&$\rm DBP_{ZH-EN}$&$\rm DBP_{JA-EN}$&$\rm DBP_{FR-EN}$\\
  \toprule
  Hungarian  &43.4s&19.8s&7.6s\\
  Sinkhorn(CPU)  &6.1s&6.1s&6.2s\\
  Sinkhorn(GPU)  &1.8s&1.7s&1.8s\\
  \bottomrule
\end{tabular}
}
\end{center}
\caption{Time costs of Hungarian and Sinkhorn.}\label{hvst}
\end{table}

\textbf{SEU vs. PARIS}.
As mentioned in Section \ref{sec:intro}, a recent summary \cite{DBLP:conf/coling/ZhangLCCLXZ20} notes that several "advanced" EA methods are even beaten by the conventional methods.
To make this study more comprehensive, we also compare SEU against a representative conventional method PARIS \cite{DBLP:journals/pvldb/SuchanekAS11} in Figure \ref{fig:pairs}, which is a holistic unsupervised solution to align KGs based on probability estimates.
Since PARIS may not always output a target entity for every source entity, we use the F1-score as the evaluation metric to deal with entities that do not have a match.
In our method, the F1-score is equivalent to $Hits@1$.
Consistent with Zhang's summary, PARIS is better than most GNN-based EA methods.
On the other hand, SEU outperforms PARIS significantly on these public datasets except for $\rm DBP_{ZH-EN}$.

\textbf{Hungarian vs. Sinkhorn}
Table \ref{hvsp} reports the performances of SEU(w+c) with the Hungarian algorithm and Sinkhorn operation, respectively.
Theoretically, the Hungarian algorithm could generate the optimal solution precisely, while the Sinkhorn operation can only generate an approximate solution.
Therefore, the Hungarian algorithm is always slightly better, but the performance gap is relatively small.
Furthermore, we list the time costs of these two algorithms in Table \ref{hvst}.
We observe that the time costs of the Hungarian algorithm are unstable, which depend on the dataset.
Meanwhile, the time costs of the Sinkhorn operation are much more stable.
Because the Sinkhorn operation is completely parallelizable, its time costs could be further reduced by the GPU.
In general, the Sinkhorn operation is more suitable for large-scale EA because of its higher efficiency.

\begin{table}[t]
\begin{center}
\resizebox{1\linewidth}{!}{
\renewcommand\arraystretch{0.8}
\begin{tabular}{c|cccc}
  \toprule
   Method&\textbf{DBP15K}&\textbf{SRPRS}\\
  \toprule
  GCN-Align \cite{DBLP:conf/emnlp/WangLLZ18}&103&87\\
  MuGNN \cite{DBLP:conf/acl/CaoLLLLC19}&3,156&2,215\\
  BootEA \cite{DBLP:conf/ijcai/SunHZQ18}&4,661&2,659\\
  MRAEA \cite{DBLP:conf/wsdm/MaoWXLW20}&3,894&1,248\\
  GM-Align \cite{DBLP:conf/acl/XuWYFSWY19}&26,328&13,032\\
  RDGCN \cite{DBLP:conf/ijcai/WuLF0Y019}&6,711&886\\
  HGCN \cite{DBLP:conf/emnlp/WuLFWZ19}&11,275&2,504\\
  \textbf{SEU(CPU)} &\textbf{22.1}&\textbf{13.8}\\
  \textbf{SEU(GPU)} &\textbf{16.2}&\textbf{9.6}\\
  \bottomrule
\end{tabular}
}
\end{center}
\caption{Time costs of EA methods (seconds).\protect\footnotemark}\label{tabel:time}
\vspace{-1em}
\end{table}

\textbf{Overall Time Efficiency}
We specifically evaluate the overall time costs of some EA methods and report the results in Table \ref{tabel:time}.
It is obvious that the efficiency of SEU far exceeds all advanced competitors.
Typically, existing GNN-based methods require forward propagations on every batch, and the convergence of models usually requires hundreds of batches.
Since SEU does not have any trainable parameters, it only requires forward propagation once, enabling SEU to achieve such acceleration.

\subsection{Auxiliary Experiments}
To explore the behavior of SEU in different situations, we design the following experiments:

\begin{figure}[t!]
  \centering
  \includegraphics[width=0.9\linewidth]{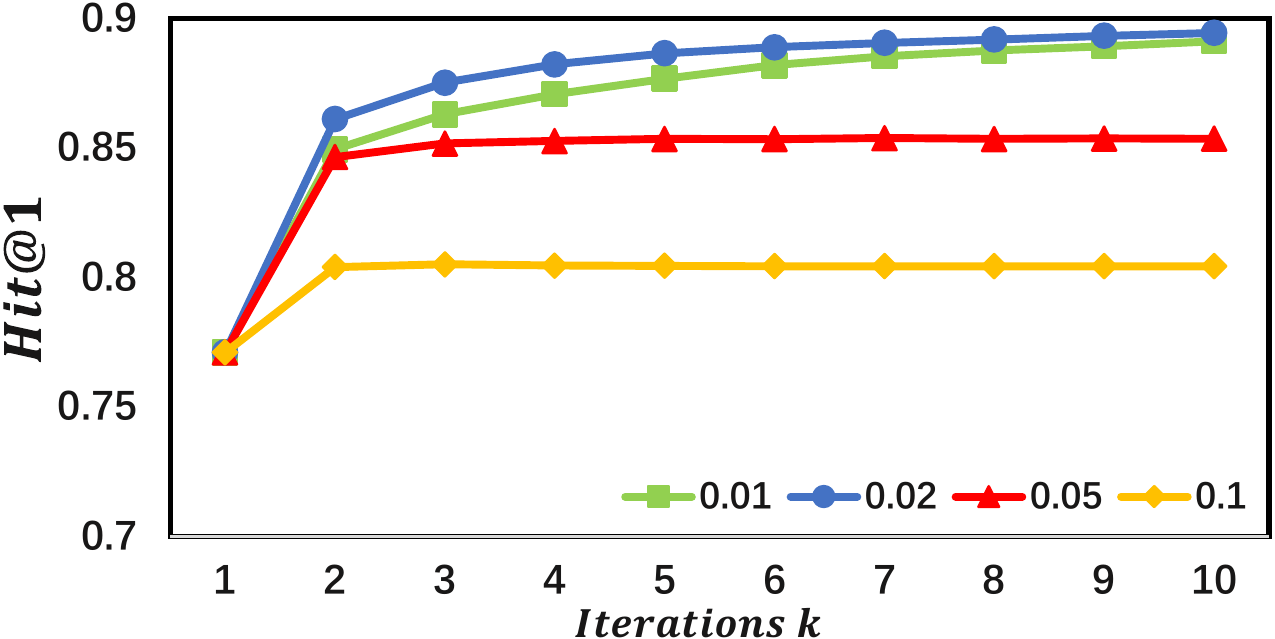}
  \caption{$Hits@1$ on $\rm DBP_{ZH-EN}$ with different $\tau$.}\label{fig:hvsp}
\end{figure}

\textbf{Temperature $\bm \tau$}.
Similar to the temperature $\tau$ in the softmax operation, $\tau$ in the Sinkhorn operation is also used to make the distribution closer to one-hot.
With the remaining config unchanged, we set $\tau$ with different values and report the corresponding performances of SEU(w+c) on $\rm DBP_{ZH-EN}$ in Figure \ref{fig:hvsp}.
If we choose an appropriate $\tau$, the Sinkhorn algorithm will converge quickly to the optimal solution.
But if $\tau$ is set too large, the algorithm will fail to converge.

\begin{figure}[t!]
  \centering
  \includegraphics[width=0.9\linewidth]{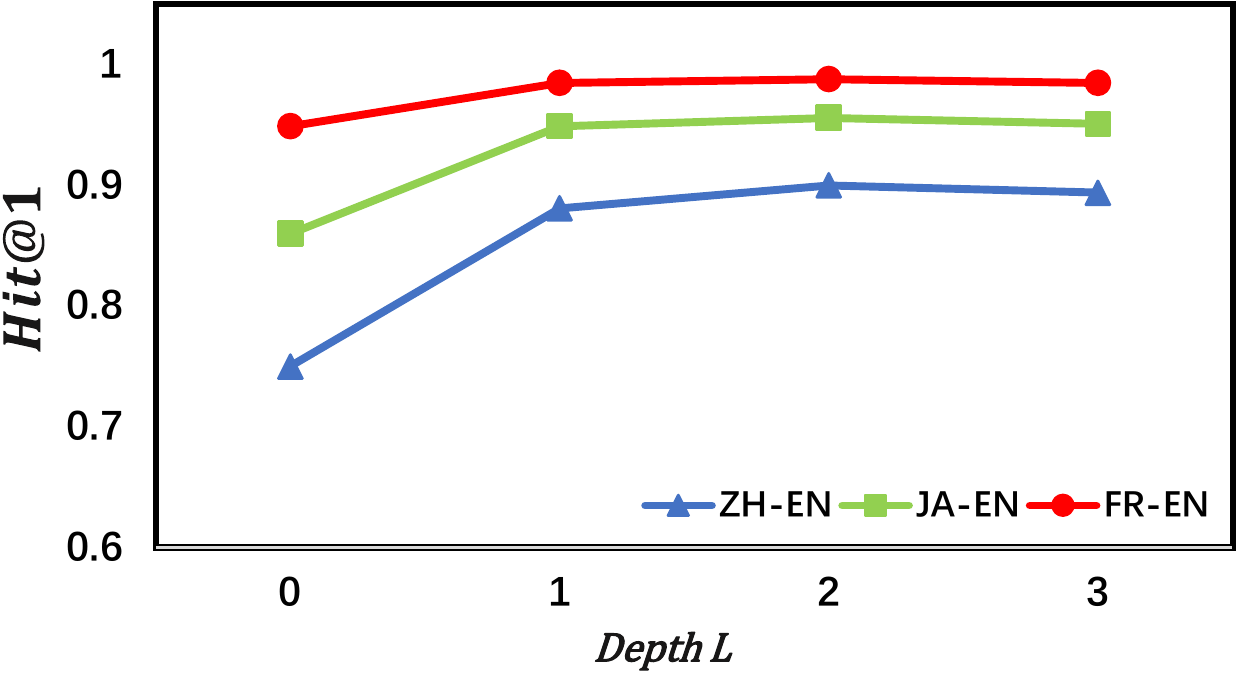}
  \caption{$Hits@1$ with different depths $L$.}\label{fig:depth}
\end{figure}

\textbf{Depth $\bm L$}.
For depth $L$, we list the experimental results in Figure \ref{fig:depth}.
In particular, $L = 0$ is equivalent to aligning entities only according to their own features without the neighborhood information.
SEU(w+c) with $L=2$ achieves the best performance on all subsets of $\rm DBP15K$, which indicates the necessity of introducing neighborhood information.
Similar to GNN-based EA methods, SEU is also affected by the over-smoothing problem.
When stacking more layers, the performances begin to decrease slightly.

\begin{table}[t]
\centering
\resizebox{1.0\linewidth}{!}{
\renewcommand\arraystretch{1.3}
\begin{tabular}{l|cccccc}
\toprule
\multirow{2}{*}{Method} & \multicolumn{2}{c}{$\rm{DBP_{ZH-EN}}$} & \multicolumn{2}{c}{$\rm{DBP_{JA-EN}}$} & \multicolumn{2}{c}{$\rm{DBP_{FR-EN}}$}\\
& Hits@1 & MRR & Hits@1 & MRR & Hits@1 & MRR\\
\toprule
 $\bm A$ &0.890&0.915&0.952&0.965&0.985&0.989\\
 $\bm A_r$ &0.891&0.916&0.953&0.966&0.985&0.988\\
 $\bm A_L$ &0.887&0.912&0.953&0.965&0.984&0.987\\
 $\bm A_{rel}$&\textbf{0.900}&\textbf{0.924}&\textbf{0.956}&\textbf{0.969}&\textbf{0.988}&\textbf{0.992}\\
\bottomrule
\end{tabular}
}
\caption{Performances with different types of adjacency matrices $\bm A$.}
\label{table:adj}
\end{table}

\textbf{Adjacency matrix $\bm A$}.
To distinguish different relation types in KGs, we adopt a simple strategy to generate the relational adjacency matrix $\bm A_{rel}$.
Table \ref{table:adj} reports the performances of SEU(w+c) with different types of adjacency matrices.
$\bm A$ is the standard adjacency matrix, $\bm A_r = \bm D^{-1}\bm A$ is the equal probability random walk matrix and $\bm A_L =\bm I - \bm D^{-1/2}\bm A\bm D^{-1/2}$ is the symmetric normalized Laplacian matrix.
The experimental results show that $\bm A_{rel}$ achieves the best performance across all these three subsets.

\subsection{Discussion}
\label{sec:dis}
From the experimental results, we observe that the supervised EA methods are even beaten by the unsupervised methods.
In this section, we propose a hypothesis that the reason behind this counter-intuitive phenomenon is potential over-fitting.

As mentioned in Section \ref{baseline}, existing EA methods could be divided into structure-based and text-based according to the input features.
The only difference between them is that the structure-based methods use randomly initialized vectors as the entity features, while the text-based methods use pre-mapped textual features as the inputs.
Let us consider the vanilla GCN as a sample:
\begin{equation}
    \bm H^{l+1} = \sigma(\bm A_L\bm H^l\bm W^l)
\end{equation}
where $\sigma$ represents the activation function.
For the structure-based methods, since the input features $\bm H$ and the transformation matrix $\bm W$ are both randomly initialized, they could be simplified into one matrix, i.e., $\bm H^{l+1} = \sigma(\bm A_L \bm H^{l})$.
This idea has been proved by many structure-based EA methods \cite{DBLP:conf/acl/CaoLLLLC19,DBLP:conf/wsdm/MaoWXLW20}, which propose to diagonalize or remove the transformation matrix $\bm W$.
In this situation, GCN is reduced to a simple fully connected neural network with adjacency matrices as its input features.
The essence of structure-based EA methods is to map the features of adjacency matrices into a unified vector space.
Therefore, these structure-based EA methods require supervised data to learn the parameters.

As for the text-based EA methods, the textual features of entities have already been pre-mapped into a unified semantic space by machine translation or cross-lingual word vectors.
Therefore, these text-based EA methods are equivalent to further fitting these pre-mapped features on a few aligned entity pair seeds, which could cause potential over-fitting.
Considering that we could directly align entities as an assignment problem, it is unnecessary to further fit entity features via neural networks.

As a simple unsupervised method, our proposed SEU achieves excellent performances on several EA datasets, which confirms the above analysis from the empirical side.
It is noted that this section only proposes a possible explanation, not rigorous proof.
We will continue to explore in this direction.

\section{Conclusion}
In this paper, we successfully transform the cross-lingual EA problem into the assignment problem.
Based on this finding, we propose a frustratingly Simple but Effective Unsupervised EA method (SEU) without neural networks.
Experiments on widely used public datasets indicate that SEU outperforms all advanced competitors and has high efficiency, interpretability, and stability.
\bibliography{anthology}
\balance
\bibliographystyle{acl_natbib}

\end{document}